\documentclass[sigconf]{acmart}

\copyrightyear{2020} 
\acmYear{2020} 
\setcopyright{acmcopyright}\acmConference[KDD '20]{Proceedings of the 26th ACM SIGKDD Conference on Knowledge Discovery and Data Mining}{August 23--27, 2020}{Virtual Event, CA, USA}
\acmBooktitle{Proceedings of the 26th ACM SIGKDD Conference on Knowledge Discovery and Data Mining (KDD '20), August 23--27, 2020, Virtual Event, CA, USA}
\acmPrice{15.00}
\acmDOI{10.1145/3394486.3403134}
\acmISBN{978-1-4503-7998-4/20/08}

\usepackage{amsmath,amsfonts,amssymb,amsthm,commath}
\usepackage{enumerate}
\usepackage{framed}
\usepackage{xspace}
\usepackage{microtype}
\usepackage{nicefrac}

\usepackage[utf8]{inputenc}
\usepackage{CJK}
\usepackage{indentfirst}
\usepackage{booktabs,color} 
\usepackage{fancyhdr}
\usepackage{graphicx}
\usepackage[tight,footnotesize]{subfigure}
\usepackage{multirow}
\usepackage{listings}
\usepackage{color}
\usepackage{caption}
\usepackage{algorithm}
\usepackage{algorithmicx}
\usepackage[noend]{algpseudocode}

\usepackage{xspace}
\usepackage{url}
\usepackage{breakurl}

\usepackage{flushend}
\usepackage{xspace}

\newtheorem{lemma}{Lemma}[section]

\newtheorem{theorem}{Theorem}[section]

\newtheorem{definition}{Definition}[section]

\newcommand{\ban}[1]{{\textsf{\textcolor{blue}{[Ban: #1]}}}}
\newcommand{\sys}{\textit{GOLD}\xspace}

\newcommand{\sun}{\mathcal{N}_{u}}
\newcommand{\slown}{\mathcal{N}_{l}}

\newcommand{\upper}{\triangle_{\text{u}}^1}
\newcommand{\lowerr}{\triangle_{l}^1}
\newcommand{\trimin}{\triangle_{\text{min}}}
\newcommand{\hatn}{\hat{\mathcal{N}}}
\newcommand{\dia}{\Diamond^{1}}
\newcommand{\hide}[1]{}


\begin{document}

\fancyhead{}
\title{Generic Outlier Detection in Multi-Armed Bandit}

\author{Yikun Ban}
\email{yikunb2@illinois.edu}
\affiliation{%
  \institution{University of Illinois at Urbana-Champaign}
}

\author{Jingrui He}
\email{jingrui@illinois.edu}
\affiliation{%
  \institution{University of Illinois at Urbana-Champaign}
}


\begin{abstract}
In this paper, we study the problem of outlier arm detection in multi-armed bandit settings, which finds plenty of applications in many high-impact domains such as finance, healthcare, and online advertising. For this problem, a learner aims to identify the arms whose expected rewards deviate significantly from most of the other arms. Different from existing work, we target the generic outlier arms or outlier arm groups whose expected rewards can be larger, smaller, or even in between those of normal arms. To this end, we start by providing a comprehensive definition of such generic outlier arms and outlier arm groups. Then we propose a novel pulling algorithm named \sys\ to identify such generic outlier arms.
It builds a real-time neighborhood graph based on upper confidence bounds and catches the behavior pattern of outliers from normal arms.
We also analyze its performance from various aspects.
In the experiments conducted on both synthetic and real-world data sets, the proposed algorithm achieves 98\% accuracy while saving 83\% exploration cost on average compared with state-of-the-art techniques.

\end{abstract}

\maketitle

\section{Introduction}

The Multi-Armed Bandit (MAB) problems have been extensively studied with various applications such as online recommendation \cite{2010licontextual, 2019dynamic, 2016optimal}, crowdsourcing ~\cite{2009efficiently, 2015statistical,DBLP:conf/kdd/ZhouNH18}, etc. 
In the classic framework, at each round of the game, a learner is faced with a set of arms, pulls an arm, and receives a sample reward from an unknown distribution associated with it. We refer to the unknown expectation of distribution as the \emph{expected reward} of an arm. With the objective of maximizing the cumulative reward, existing approaches aim to identify a set of arms with the largest expected rewards (named top-$k$ arm identification problem) \cite{agresti1998approximate,2002finite,2010best,2014lil,2002using}. 
Another line of works investigates the thresholding bandit Problem ~\cite{2016optimal, 2019thresholding}. The learner needs to find all the arms whose mean rewards are above a given threshold.

In this paper, we focus on outlier arm detection in the MAB setting. An arm is thought of as an outlier when its expected reward deviates from most of the other arms.
Identifying outlier arms can be applied to many applications. Consider the medical tests where a disease can be modeled as an arm and observe the degree of presence on a patient. We would like to detect the arm with an exceptionally high expected reward compared to other arms, in order to identify the disease having a significantly higher presence with respect to a certain biomarker \cite{2017outlierarm}. In the online recommendation, an article can be modeled as an arm \cite{2010licontextual}. Finding the article whose expected reward is higher than the rest enables us to analyze the article with high popularity and recommend similar articles to certain users. In the meanwhile, identifying the article with a much lower expected reward compared to others also helps us to replace it with alternative articles.

\begin{figure*}[h]
\centering
\includegraphics[width =1.0 \textwidth ]{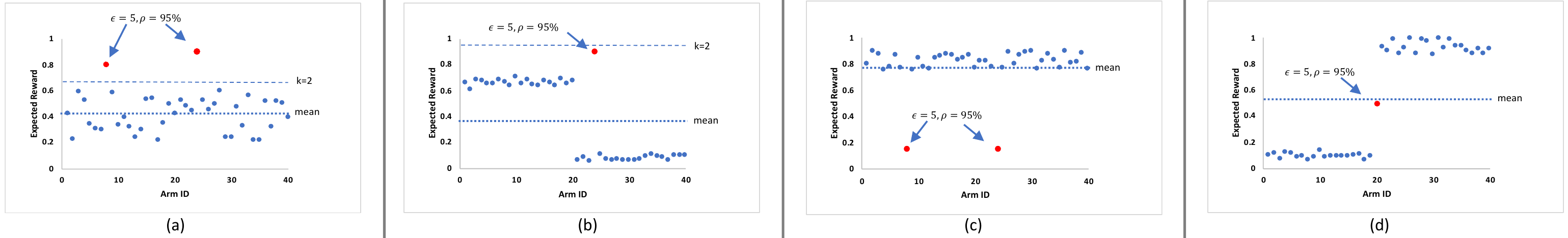}
 \vspace{-1.0em}
\caption{The outlier arms (red points) in four different global distribution can be formulated as ($\epsilon=5, \rho = 95\%$)-outlier arms. For ($k=2$)-sigma rule, it works well on the normal distribution, as shown in (a). However, it cannot identify outlier arms correctly in other distributions, as shown in (b), (c), and (d). } \label{sample}
 \vspace{-1.5em}

\end{figure*}

The outlier arm detection problem was first studied by \cite{2017outlierarm}. It classifies the arms whose expected rewards are \emph{above} a threshold as the outliers. The threshold is defined by a general statistical rule, the mean of expected rewards of all arms plus $k$ times standard deviation, called the \textbf{$k$-sigma rule}. 
Unfortunately, there are two major drawbacks of using the $k$-sigma rule to identify outlier arms. 
In this paper, we refer to the distribution of expected rewards of all arms as \emph{global distribution}, which is the unknown prior knowledge.

First, it cannot detect the outlier arms whose expected reward is not high while being far away from most of the other arms. 
Figure \ref{sample} (c) and (d) provide two intuitive examples. The arm with an unusually low expected reward compared to most of the others should be regarded as an outlier. Moreover, in a skewed global distribution, it is also possible that the expected reward of an outlier arm is close to the median. However, according to the $k$-sigma rule, \cite{2017outlierarm} is only able to detect the arms with unusually high expected rewards, and cannot be used to detect the more generic type of outlier arms.

Second, the optimal value of $k$ is difficult to choose as prior knowledge, including the global distribution, is unknown. Using the same $k$ may result in varying performance on different data sets. In Figure \ref{sample} (a), the expected value of $k$ is $2$, such that two arms denoted as two red points are identified as outlier arms. On the contrary, if $k=1$, then many normal arms (blue points) will be identified as outlier arms; if $k=3$, then one outlier arm (the topmost red point in Figure \ref{sample} (a)) will be missed. In Figure \ref{sample} (b), the outlier arm cannot be detected using the $k$-sigma rule if $k$ is set to $2$. But if we set $k=3$ or $k=1$, false positives or negatives will be introduced. Therefore, without the prior knowledge of the global distribution, it can be difficult to pick the optimal value for $k$.

To address the above challenges, we propose a more comprehensive definition for generic outlier arms: instead of only focusing on the arms with an exceptionally high expected reward, an outlier arm should have an expected reward that is `unusually far' from `most' of the other arms. 
In particular, we use two parameters $\epsilon$ and $\rho$ to quantify these two factors (see Definition \ref{def:single}). 
$\epsilon$ is to define a reasonable distance threshold, which guarantees that the shortest distance between the outlier and normal arms is much longer than the neighborhood distance of each normal arm.
$\rho$ is to restrict the smallest proportion of normal arms. For example, $\rho = 95\% $ means that the outlier arm is far away from at least $95\%$ of all the arms. In practice, the optimal values of $\epsilon$ and $\rho$ are easier to pick even if no prior knowledge is known. For example, Figure~\ref{sample} shows the outlier arms (red points) according to our proposed definition with the same pair of $\epsilon$ and $\rho$ in four different global distributions, which is consistent with our intuition.


Existing methods such as for top-$k$ arm identification  ~\cite{2013multiple, 2012pac,2016complexity}  or thresholding bandit problem ~\cite{2019thresholding, 2016tight}  cannot directly be applied to this problem, as our goal is to detect the arms with the deviating rewards instead of maximizing the rewards of a set of arms. 
Given $\epsilon$ and $\rho$ in our proposed definition, one wants to identify outlier arms that truly satisfy the criterion with a high probability. However, given a set of arms, the specific number of outlier arms is usually unknown. Moreover, even given a set of outliers, the arm deviates most should have the first priority of being inspected when the ground truth is not provided in practice.
Therefore, instead of identifying a subset of arms as outliers, we propose to rank order all the arms such that the outlier arms are ranked higher than the normal arms.

In this paper, we make the following major contributions:

\begin{enumerate}
    \item We propose a comprehensive definition for outlier arms and formally define the ($\epsilon, \rho$)-outlier arm detection problem within the MAB framework.
    \item We propose a novel pulling algorithm, Generic OutLier Detection (\sys), to identify outlier arms, given ($\epsilon, \rho$).
    \item We prove the correctness of \sys with the lower bounds of $\epsilon$ and $\rho$, and the upper bound on the max number of pulls.
    \item We evaluate our algorithm on both synthetic and real-world data sets: \sys can achieve nearly perfect accuracy on all the data sets while reducing the cost by 83\% on average compared with state-of-the-art techniques. 
\end{enumerate}

The rest of the paper is organized as follows. After a brief discussion of the related work in Section 2, we formally present the problem definition in Section 3. The proposed algorithm is introduced in Section 4, followed by the theoretical analysis in Section 5. In Section 6, we present the experimental results on both synthetic and real-world data sets, before we conclude the paper in Section 7.

\section{Related Work} \label{sec:relatedwork}

In this section, we briefly review the related works on MAB problems and outlier detection.

\noindent \textbf{MAB problems}. \xspace  A key example of solving exploration versus exploitation dilemma is the multi-armed bandit problem first proposed by \cite{1933likelihood}.
The goal of it is to minimize the cumulative regret, the difference between the average rewards of the optimal arms and the average rewards of selected arms. A large branch of works focuses on the best arm identification with the fixed confidence and fixed budget, i.e., find the best arm with a probability above a threshold below a certain number of pulls \cite{2010best, 2006action,2012best,2016optimal, 2002finite,2009pure,2011pure,2016tight}. Another line of works eliminates the suboptimal arms sequentially with certain confidence \cite{1964sequential, 1994hoeffding,2002pac}. Then, this problem has been extended to the Top-K arm problems, identify a set of K arms with the largest expected rewards ~\cite{2013multiple, 2012pac,2016complexity}. In the pure exploration setting, the learner aims at making the best exploration with a small cost, in order to optimize the performance of some decision-making tasks \cite{2011pure,2016optimal}. Finding optimal arms is different from identifying the outlier arms because the optimal arms are unnecessary to be outliers.
The thresholding bandit problem is to find a set of arms whose expected reward is larger than a threshold ~\cite{2019thresholding, 2016tight}. However, these algorithms cannot apply to the outlier arm detection because the expected rewards of all arms are unknown such that the threshold is not able to define.

\noindent \textbf{Outlier detection}. \xspace  The problem of outlier detection has been extensively studied in the data mining field~\cite{2009anomaly, 2015graph}.  
According to ~\cite{2011data}, "An outlier is a data object that deviates significantly from the rest of the objects." 
The approaches of outlier detection can be divided into statistic-based ~\cite{2005combining, 2008statistical}, graph-mining-based~\cite{2019no, zhang2017hidden}, nearest-neighbor-based~\cite{1999elliptical,2006spatial}, semi-supervised-based ~\cite{zhou2018sparc, zhou2015muvir,DBLP:conf/www/ZhouNMFH20}, and so on.
However, most of the existing approaches for outlier detection rely on the observed data rather than the online learning process.
Among them, the nearest neighbor-based or distance-based outlier \cite{2011data} is most closed to the definition of $(\epsilon, \rho)$-outlier arm.  For example, a score of an outlier is measured by its distance to its k-th nearest neighbor, called k-th Nearest Neighbor \cite{2009anomaly}.
The difference between them is that $(\epsilon, \rho)$-outlier considers the distance from the outlier to the normal arms instead of its local neighbors.

\noindent \textbf{Outlier detection in MAB}. \xspace The first work studying the outlier detection in MAB framework is the recent work~\cite{2017outlierarm}. They proposed two algorithms to identify the outlier arms using the $k$-sigma rule. One algorithm, named RR, uses the round-robin method to pull arms. Another algorithm, named WRR, uses the weighted round-robin way to save the pulling cost. Furthermore, They provide a theoretical guarantee for the correctness of the proposed algorithm with probability at least 1- $\delta$. However, as mentioned before, their algorithms only are able to detect a narrow range of outlier arms, i.e., the arms with exceptional high expected reward. Also, $k$ cannot be set in a heuristic way in the MAB framework.  Instead of applying $k$-sigma rule, in this paper, we provide a novel formula for generic outlier arms and the corresponding detection algorithm.

\section{Problem Definition}  \label{sec:problem}

In this section, we introduce the formal definitions of ($\epsilon, \rho$)-outlier arm, ($\epsilon, \rho$)-outlier arm group, and ($\epsilon, \rho$)-outlier arm detection.


Our goal here is to identify the outlier arms whose expected rewards are far away (significantly deviating) from most of the other arms.
Since the expected reward is a scalar, therefore, outlier arms can be categorized into two types. The first type is the arms with extreme expected rewards. In other words, the expected rewards of normal arms can be {\it either} higher {\it or} lower than these arms (e.g., Figure~\ref{sample}(a)-(c)). The second type is the arms without extreme expected rewards. However, the expected rewards of normal arms are {\it both} higher and lower than these arms (e.g., Figure~\ref{sample}(d)). Notice that the algorithm proposed in \cite{2017outlierarm} is only able to detection some outlier arms of the first type, i.e., those with unusually high expected rewards.

Therefore, we propose a comprehensive definition of the outlier arms based on the following three observations: (1) outlier arms are far from the normal arms in terms of the expected rewards;
(2) the expected rewards of normal arms are closer to the other normal arms in the local neighborhood compared to the outlier arms;
(3) the vast majority of all the arms are normal arms, and their expected rewards could be larger or smaller than the small number of outlier arms.

Let $\Psi = \{1, ..., n\}$ denote the set of $n$ arms. Each arm $i \in \Psi$ is associated with a probability distribution with an expected reward $y_i$ (unknown).
Given a subset of arms $\mathcal{N}\subset \Psi$, for each $i \in \mathcal{N}$, it typically has two nearest neighbor arms on two sides. One is the upper nearest neighbor arm with respect to arm $i$, and we denote the their distance as $\upper(i, \mathcal{N})$. Formally,
$$
\upper(i, \mathcal{N}) = 
\begin{cases}
0,    &\text{if} \  \not \exists i' \in \mathcal{N}, y_{i'} > y_i;\\
 \min_{i' \in \mathcal{N}, y_{i'} > y_i} (y_{i'} - y_i), & \text{otherwise.} \\
\end{cases}
$$
The other one is the lower nearest neighbor arm with respect to arm $i$, and we denote their distance as $\lowerr(i, \mathcal{N})$. Formally,
$$
\lowerr(i, \mathcal{N}) = 
\begin{cases}
0,    &\text{if} \  \not \exists i' \in \mathcal{N}, y_{i'} < y_i;\\
 \min_{i' \in \mathcal{N}, y_{i'} < y_i} (y_{i} - y_{i'}), & \text{otherwise.} \\
\end{cases}
$$
For simplicity, we assume that all arms have distinct expected rewards, i.e.,$\forall i \not = j$, $y_i \not = y_j$.

Then, given an arm $i$ and a subset $\mathcal{N}$, we define the \emph{local distance of $1$-st upper and lower nearest neighbors} $\dia$ as follows,
$$
\dia(i, \mathcal{N}) = \max \left \{ \upper(i, \mathcal{N}), \lowerr(i, \mathcal{N})  \right \}.
$$
$\dia$ measures the local proximity of an arm. For brevity, we call it \textbf{neighborhood distance}. In fact, we can extend $\dia$ to the local distance of $k$-th upper and lower nearest neighbors denoted as $\Diamond^k$. The extension to the $k$-th upper and lower nearest neighbors is beyond the scope of the current paper, and is left for future work.


Based on the above notation, we first provide the definition with respect to a single outlier arm.

\begin{definition}\label{def:single} \emph{(($\epsilon, \rho$)-outlier arm).} \xspace
Given an arm $j \in \Psi$, then $j$ is an $(\epsilon, \rho) $-outlier arm in $\Psi$, if $\exists \sun \subseteq \{i \in \Psi: y_i > y_j\} , \exists\slown \subseteq \{i \in \Psi: y_i < y_j\} $ that satisfy the following two constraints:
\begin{displaymath}\label{eq:first}
\hspace{3mm}\textrm{\bf Constraint (1):} \begin{cases}
a) \    \triangle_{\text{min}}(j,  \sun ) >  (1+\epsilon) 
 \dia(i, \sun), \  \forall i \in \sun  \\
b)  \    \triangle_{\text{min}}(j,  \sun ) >  (1+\epsilon)
 \dia(i, \slown) , \ \forall i \in \slown \\
c)  \    \triangle_{\text{min}}(j, \slown) >  (1+\epsilon)  \dia(i, \sun),  \ \forall i \in \sun  \\
d)   \    \triangle_{\text{min}}(j, \slown) >  (1+\epsilon)  \dia(i, \slown),  \ \forall i \in \slown,\\
\end{cases}
\end{displaymath}
where $\epsilon > 0$ and $\triangle_{\text{min}}(j, \sun) = \min_{i' \in \sun} |y_j - y_{i'}|$;
\begin{displaymath}\label{eq:first}
\hspace{3mm}\textrm{\bf Constraint (2):} \begin{cases}
|\sun| + |\slown| > \rho \times n \\
 |\sun|=0  \ \text{or} \  |\sun|> (1-\rho)\times n \\
|\slown|=0  \ \text{or} \ |\slown|> (1-\rho)\times n,
\end{cases}
\end{displaymath}
where $1> \rho > 0.5$. 
\end{definition}

According to Constraint (1),
$\trimin$ is the shortest distance between an outlier arm and a normal arm set, such as $\trimin(j, \sun)$. Since $\sun$ and $\slown$ are separated by $j$, Constraints $a)$ and $b)$ guarantee that $\trimin(j, \sun)$ is $(1+\epsilon)$ longer than the neighborhood distance of each normal arm in $\sun \cup \slown$.
Similarly, Constraints $c)$ and $d)$ ensure that $\trimin(j, \slown)$ is $(1+\epsilon)$ longer than the neighborhood distance of each normal arm in $\sun \cup \slown$.

According to Constraint (2),
$\rho$ represents the assumed proportion of normal arms in $\Psi$.
Intuitively, normal arms should be restricted by $\rho$, i.e.,  $|\sun| + |\slown| > \rho \times n$. In contrast, the proportion of outlier arms should be lower than $1-\rho$. For the first type of outlier arms with extreme expected rewards, only one normal arm group exists above $j$ or below $j$, which can be represented by $|\sun| > \rho \times n \ \& \ |\slown| = 0 $ or  $|\sun| = 0 \ \& \ |\slown| > \rho \times n $. For the second type of outlier arms without extreme expected rewards, there are at least two normal arm groups in their neighborhood. This case can be modeled by $|\sun| + |\slown| > \rho \times n \ \& \ |\sun| >(1-\rho)\times n \ \& \  |\slown| >(1-\rho)\times n$, since the normal arm groups should be larger than the assumed number of outlier arms.
In practice, the number of outlier arms is often much smaller than the number of normal arms, i.e., $1 > \rho > 0.5$. 

In real applications, in addition to individual outlier arms, some ourlier arms might form small groups based on their expected rewards. Next we further consider an $(\epsilon, \rho)$-outlier arm group $\hat{\mathcal{N}}$.

\begin{definition}\label{def:group} \emph{(($\epsilon, \rho$)-outlier arm group).} \xspace
Given three sets of arms $\hatn \subset \Psi$,  $\sun = \{i \in \Psi: y_i > \max_{j \in \hatn} y_{j}\}$, and $\slown = \{i \in \Psi: y_i < \min_{j \in \hatn}y_{j}\}$ that satisfy $ \hatn \cup \sun \cup \slown = \Psi$, then $\hatn$ is an $(\epsilon, \rho) $-outlier arm group with respect to $\sun$ and $\slown$ if the following two constraints are satisfied:
\begin{displaymath}\label{eq:first}
\hspace{3mm}\textrm{\bf Constraint (1):}
\begin{cases}
   \triangle_{\text{min}}(j,  \sun ) >  (1+\epsilon)\dia(i, \sun), \forall  j \in \hatn,  \forall i \in \sun \\
  \triangle_{\text{min}}(j,  \sun ) >  (1+\epsilon) \dia(i, \slown), \forall j \in \hatn,  \forall i \in \slown \\
  \triangle_{\text{min}}(j, \slown) >  (1+\epsilon) \dia(i, \sun), \forall j \in \hatn, \forall i \in \sun \\
  \triangle_{\text{min}}(j, \slown) >  (1+\epsilon) \dia(i, \slown), \forall j \in \hatn, \forall i \in \slown \\
\end{cases}
\end{displaymath}
where $\epsilon > 0$ and $\triangle_{\text{min}}(j, \sun) = \min_{i' \in \sun} |y_j - y_{i'}|$;

{\bf Constraint (2):} The same as {\bf Constraint (2)} in Definition \ref{def:single}.
\end{definition}

In Figure \ref{sample}, the red points in four different global distributions can all be formulated as the $(\epsilon = 5, \rho = 0.95) $-outlier arm group.

Instead of detecting the fixed number of arms as the outlier, we aim to generate a ranked list of all arms denoted by $\Omega$, to guarantee that higher-ranked arms are more likely to be outlier arms.

Based on the above definitions, we now give the following formal problem definition of ($\epsilon, \rho$)-outlier arm detection.
\smallskip

\noindent \textbf{Problem Definition.} \emph{(($\epsilon, \rho$)-outlier arm detection)}. \xspace
Given $\epsilon, \rho$,  and $\Psi$ where $\epsilon >0$ and $1 > \rho  > 0.5$, identify a ranked list of all arms in $\Psi$ denoted by $\Omega$, such that for any $\hatn \subset \Psi$, if $\hatn$ is an $(\epsilon, \rho)$-outlier arm group with respect to $\sun$ and $\slown$, it satisfies:
\begin{equation} \label{eq:condi}
\begin{cases}
 \forall j \in \hatn, \forall i \in \sun,  \text{rank}(j) <  \text{rank}(i),\\
  \forall j \in \hatn, \forall i \in \slown,  \text{rank}(j) <  \text{rank}(i).
\end{cases}
\end{equation}
where \text{rank}($j$) is the rank of $j$ in $\Omega$.
\smallskip

Note that there may be multiple possible combinations of $\hatn$, $\sun$, and $\slown$ for dividing $\Psi$. For example, if $\hat{\mathcal{N}} =\{j_1\}$ is the $(\epsilon, \rho)$-outlier arm group with respect to   $\sun = \Psi-\{ j_1\}$ and $|\slown| = 0$,  it should have $\forall i \in   \sun, \text{rank}(j_1) <  \text{rank}(i)$. 
Then, if $\hatn =\{j_1, j_2\}$ also is an $(\epsilon, \rho)$-outlier arm group with respect to  $\sun = \Psi-\{ j_1, j_2\}$, it should have $\forall j \in \{ j_1, j_2\},   \forall i \in   \sun, \text{rank}(j) <  \text{rank}(i)$. 
Therefore, we believe that the first arm in $\Omega$ has the top priority of being investigated. 
In other words, without the ground truth, human experts should inspect outlier arms according to $\Omega$.

In a multi-armed bandit setting, for each arm $i \in \Psi$, the value of $y_i$ is unknown. In the round $T$, the learner needs to pull an arm $i$ and obtain the $j$-th reward $x_i^{(j)}$. $T$ also is the number of pulls in total so far.  Let $m_i$ be the number of pulls on arm $i$ so far, $i=1,\ldots,n$, (i.e., $T = \sum_{i = 1}^n m_i$). $\hat{y}_i$ is the empirical estimate of $y_i$ as follows:
\begin{displaymath}
 \hat{y}_i = \frac{1}{m_i} \sum_{j=1}^{m_i} x_i^{(j)}. 
\end{displaymath}

 Given a set of arms $\Psi$ with $\epsilon, \rho,$ and $\delta$ ($\delta$ usually is a small constant),  our objective is to design a pulling algorithm, such that the returned $\Omega$ satisfies the criteria of ($\epsilon, \rho$)-outlier arm detection with probability at least $1-\delta$. 
 

\section{Proposed Pulling Algorithm}
Given the problem definition of ($\epsilon, \rho$)-outlier arm detection, in this section, we propose a novel pulling algorithm named Generic OutLier Detection (\sys).

\subsection{Preliminaries}

Before presenting \sys, we first introduce several key components.

\begin{definition} \emph{(Neighbor Arms)}. \xspace \label{def:neigh}
Given two arms $i, j \in \Psi$, in the round $T$, they are considered as neighbor arms if  
\begin{equation}\label{eq:neigh}
|\hat{y}_i - \hat{y}_j| \leq b \left[ \beta_{i}(m_i, \delta') + \beta_{j}(m_j, \delta') \right ],
\end{equation}
where $b$ is a coefficient function with regard to $\epsilon$ and  $\beta_{i}(m_i, \delta')$ is the Upper Confidence Bound (UCB) of $\hat{y}_i$.
\end{definition}

In our algorithm, as we will shown in the next section (Lemma \ref{lemma:event} and \ref{the:split}),  $b$ and $\delta'$ should be set to:
\begin{equation}
b = \frac{1 +e^{\frac{1}{16}}+ \epsilon}{1 -e^{\frac{1}{16}} + \epsilon} \  \  \text{and}  \ \ \delta' = \frac{6\delta}{\pi^2 n T^2},
\end{equation}
where $\epsilon > e^{\frac{1}{16}} -1$ to ensure $b>0$.

The UCB of $\hat{y}_i$ is set differently based on the prior knowledge of reward distributions which can be divided into two classes:

\begin{enumerate}
\item {\bf Bounded Distribution.} Suppose the reward is bounded by $[c, d]$ for each arm,
$R = d-c$. According to \textit{Hoeffding’s inequality}, the UCB of $\hat{y}_i$ is
 \vspace{-0.5em}
\begin{equation} \label{eq:boundeddis}
\beta_{i}(m_i, \delta') = R\sqrt{\frac{-\log{\delta'}}{2m_i}}.
\end{equation}

\item {\bf Bernoulli Distribution.} In this case, the reward obtained from an arm is either 0 or 1. We follow the confidence bound presented in \cite{agresti1998approximate, 2017outlierarm}, which is defined as:

\begin{displaymath}
\beta_{i}(m_i, \delta') = Z \sqrt{\frac{\tilde{p}(1- \tilde{p})} { m_i}  } \   \&  \ \tilde{p} = \frac{\tilde{m}_i^{+} + \frac{Z^2}{2}}{m_i + Z^2}  \ \& \  Z = \text{erf}^{-1} (1-\delta'),   
\end{displaymath}

where $\tilde{m}_i$+ is the number of rewards that equal to 1 among $m_i$ rewards, and $Z$ is the value of the inverse error function.
\end{enumerate}

Notice that $\delta'$ and $\beta_{i}(m_i, \delta')$ will be updated in each round accord to $T$. Therefore, the neighbor relationship among arms may be different in each round. We use a graph to represent the neighborhood relationship of all arms $\Psi$.

\begin{definition} \emph{(Neighborhood Graph)}. \xspace
In round $T$, the neighborhood graph, denoted as $G = (\Psi, E)$, is formed by $\Psi$ where each node represents an arm and an unweighted and undirected edge exists between any pair of arms if they are neighbor arms.
\end{definition}

In each round, we divided arms into groups as the following.


\begin{definition} \emph{(Arm Community)}.  \xspace
In round $T$,  the arm communities are formed by the connected components of $G$, denoted as $\mathbf{M} = \{\mathcal{M}_1, ..., \mathcal{M}_{k}\}$ where $\mathcal{M}_i$ is an arm community formed by a connected component, $i=1,\ldots,k$, and the size of $\mathcal{M}_i$ denoted by $|\mathcal{M}_i|$ is the number of nodes in this connected component.
\end{definition}

With the above preliminaries, we detail \sys in the following section.



\subsection{\sys Algorithm}

\sys is a pulling strategy and stops when it is confident about the outlier arms.
First, we provide its initial and terminal statuses. 

\noindent \textbf{Initial Status.} \xspace 
After pulling each arm once, $G$ becomes a \emph{complete} graph, i.e., for any two arms $i,j \in \Psi$, an edge exists between them as they are neighbor arms. It is simple to prove that $i$ and $j$ are neighbor arms when $m_i = m_j = 1$.

\noindent \textbf{Terminal Statuses.} \xspace
The terminal statuses include the \textit{terminal status of an arm} and the \textit{terminal status of the algorithm}.
Given any two arms $i, j \in \Psi $, $i$ and $j$ will not be neighbor arms in the end if we can pull each arm infinite times, based on Theorem \ref{theo:neighborar}.

Without any restriction, $G$ will end as a graph where each node is isolated. With $\rho$, we define the terminal statuses as follows. 

The terminal status of an arm is defined as the round when its community whose size is smaller than $n \times (1-\rho)$.  The set $\widehat{\Psi}$ is to keep the arms that have achieved the terminal status during the pulling process.

The terminal status of the algorithm is defined as the round when the number of arms that have already achieved the terminal status is not smaller than $n \times (1-\rho)$, i.e., $|\widehat{\Psi}|\geq  n \times (1-\rho)$.


\begin{algorithm}[t]
\renewcommand{\algorithmicrequire}{\textbf{Input:}}
\renewcommand{\algorithmicensure}{\textbf{Output:}}
\caption{ \sys\ Algorithm }\label{alg:main}
\begin{algorithmic}[1]
\Require $\epsilon$, $\rho$, $\Psi$, $\delta$
\Ensure $\Omega$
\State $\widehat{\Psi} \leftarrow \emptyset$, $T  \leftarrow 0$, $n \leftarrow |\Psi| $ 
\State $\forall i \in \Psi, S[i] \leftarrow 0, \hat{y}_i  \leftarrow 0, m_i   \leftarrow 0, \beta_{i}(m_i, \delta')\leftarrow 0 $ 
\For{each $i \in \Psi$}
\State  pull $i$ once
\State $ T \leftarrow T + 1, \ m_i \leftarrow m_i +1$
\State updates $\hat{y}_i$, $\beta_{i}(m_i, \delta')$ 
\EndFor 
\State $G \leftarrow (\Psi, E =\emptyset)$)  
\For{ each $i, j \in \Psi$}
\If {$ |\hat{y}_i - \hat{y}_j| \leq b \left[ \beta_{i}(m_i, \delta') + \beta_{j}(m_j, \delta') \right ]$}
\State $E \leftarrow E + \{e_{ij}\}  $ \ \ \small{ \#\# create an edge between $i$ and $j$}
\EndIf
\EndFor 
\While{ $|\widehat{\Psi}| <  n \times (1-\rho)  $ }
\State $\mathcal{N} \leftarrow (\Psi - \widehat{\Psi})$
\For{ each $i$  $ \in \mathcal{N}$}
\State pull arm $i$
\State $T \leftarrow T + 1, \ m_i \leftarrow m_i +1$
\State update $\hat{y}_i$,  $\beta_{i}(m_i, \delta')$
\State $G \leftarrow$ UpdateG($G$) 
\State $\mathbf{M} \leftarrow \text{ConnectedComponents}(G)$     \ \ \small{\#\# return the communities}
\State $S , \widehat{\Psi} \leftarrow$ Update$\widehat{\Psi}$and$S$($\mathbf{M}, \widehat{\Psi},   S, T$).
\EndFor
  
\EndWhile
\State $\Omega \leftarrow$  rank $\Psi$ according to $S$
\State \textbf{Return:}  {$\Omega$}
\\
\Procedure{UpdateG}{$G = (\Psi, E)$}
\For{ each $i, j \in \Psi$}
\If {$ |\hat{y}_i - \hat{y}_j| > b \left[ \beta_{i}(m_i, \delta') + \beta_{j}(m_j, \delta') \right ]$}
\If {$e_{ij} \in E$}
\State $E \leftarrow E - \{e_{ij}\}$ \ \ \small{ \#\# delete the edge between $i$ and $j$}
\EndIf 
\EndIf
\EndFor
\State \textbf{Return:} $G$
\EndProcedure
\\
\Procedure{Update$\widehat{\Psi}$and$S$}{$\mathbf{M}, \widehat{\Psi},   S, T$}
\For{ each $\mathcal{M} \in  \mathbf{M}$}
\If {$ |\mathcal{M}| < n \times (1-\rho) $  }  
\For {each $i \in \mathcal{M}$ }
\State $S[i] \leftarrow T$
\State $\widehat{\Psi} \leftarrow \widehat{\Psi} \cup \{i\}$
\EndFor
\EndIf 
\EndFor
\State \textbf{Return:} $S$, $\widehat{\Psi}$
\EndProcedure
\end{algorithmic}
\end{algorithm}

\noindent \textbf{Ranking Process.} The goal of \sys is to return the ranked list $\Omega$. \sys maintains an $S$-score for each arm, e.g., $S[i]$ represents the $S$-score of arm $i$. In the pulling process from the initial status to the terminal status of the iteration, $S[i]$ will be updated when arm $i$ achieves the terminal status. Eventually, all the arms will be ranked according to $S$.

Next, we introduce the details of \sys as the following steps.

 \textit{Step 0: Initialization (Line 1-10 ). \xspace} Before the arm pulling iteration, some variables are required to be defined and initialized (Line 1-5). 
In Line 7-10, $G$ is built and updated, starting as a complete graph.

 \textit{Step 1: Arm Pulling (Line 11-16). \xspace}
After the initialization, we start the pulling process. In each iteration, we sequentially pull the arms that have not achieved the terminal status, denoted by $\{\Psi - \widehat{\Psi}\}$.
In the end, the iteration will stop when $|\widehat{\Psi}|\geq  n \times (1-\rho)$.

\textit{Step 2: Update $G$ (Line 17 and Line 23-28). \xspace}
In each round, we update $G$ by removing the edge if its incident arms are not neighbor arms anymore. 

\textit{Step 3: Update $\widehat{\Psi}$ and $S$ (Line 19 and 30-36) \xspace}. In each round, new communities may be formed.
We check the size of each community in each round.
For any arm that has achieved terminal status, it will be appended into $\widehat{\Psi}$, and its $S$-score will be updated by the value of the present round $T$.

\textit{Step 4: Return $\Omega$ (Line 20-21). \xspace } After ending the pulling process,  $\Omega$ is a ranked list of all arms in ascending order according to $S$-score.

\section{Theoretical Analysis}
In this section, we first provide a theoretical guarantee regarding the correctness of $\Omega$ returned by \sys. Then we show the upper bound of the number of pulls needed for \sys to terminate.

\subsection{Correctness Analysis of \sys}

\begin{theorem} \label{theo1}
Given $\Psi$ with $\epsilon, \rho,$ and $\delta$, $\Omega$ is the result returned by $\sys$. If an $(\epsilon, \rho)$-outlier arm group exists in $\Psi$ where $\epsilon > e^{\frac{1}{16}}-1$ and $\rho > 0.5$, denoted by $\hatn$ with respect to $\sun$ and $\slown$, then it has 
$$
\begin{cases}
 \forall j \in \hatn, \forall i \in \sun,  \text{rank}(j) < \text{rank}(i)\\
  \forall j \in \hatn, \forall i \in \slown,  \text{rank}(j) < \text{rank}(i),
\end{cases}
$$
where rank(j) is the rank of $j$ in $\Omega$,
with the probability at least $1-\delta$.
\end{theorem}

To prove Theorem \ref{theo1}, we first introduce Lemmas \ref{lemma:event}, \ref{lemma:mbounds}, \ref{the:split}, and \ref{lemma:group}. See proofs in Appendix. 

These lemmas are to show that the outlier arms achieve the terminal status \emph{earlier} than the normal arms with probability at least $1-\delta$. Thus, we use $S[i] \leftarrow T$ (Line 34 in \sys) to guarantee that the $S$-score of outlier arms are smaller than normal arms.

Lemma \ref{lemma:event} defines an event where the $\hat{y}_i$ for each arm is within the confidence interval in each round and proves that the probability of this event happening is at least $1 - \delta$.

\begin{lemma} \label{lemma:event}
For each arm $i \in \Psi $, given an arbitrary probability $\delta' = \delta'(T) $,  we have the confidence interval bound $\beta_{i}(m_i, \delta')$, satisfying
\begin{displaymath}
\mathbb{P}(|\hat{y}_i - y_i | >  \beta_{i}(m_i, \delta'(T))) < \delta'(T).
\end{displaymath}
Define the event
\begin{equation}
\mathcal{E} = \{  \bigwedge_{\forall T, \forall i}|\hat{y}_i - y_i| \leq \beta_{i}(m_i, \delta'(T))\}.
\end{equation}
Assume a random event  $\mathbf{I}$ denoted by an infinite sequence $[I_1, I_2, ... ]$, where  $1 \leq I_T \leq n$ represents the arm we selected in round $T$
. If we shrink $\delta'$ as the increasing of T in a proper formula, accordingly,
\begin{displaymath}
\delta'(T) = \frac{6\delta}{\pi^2 n T^2},
\end{displaymath}
then as any $\mathbf{I}$ happened, it has
\begin{displaymath}
\mathbb{P}(\mathcal{E}|\mathbf{I}) \geq 1-\delta.
\end{displaymath}
\end{lemma}

When the event $\mathcal{E}$ happens, Lemma \ref{lemma:mbounds} derives the lower and upper bounds of the number of pulls needed on two arms until they are not neighbors.

\begin{lemma} \label{lemma:mbounds}

Given two arms $i, j \in \Psi$, let $\hat{m}_i$ be the maximal number of pulls on $i$ when $i$ and $j$ still are neighbor arms. Assuming  $\hat{m} = \min\{\hat{m}_i, \hat{m}_j\} $,
with probability $1-\delta$, we have
$$
 4 \mathbf{D}_2 \log(2\mathbf{D}_2\sqrt{\frac{\pi^2n^2}{6\delta}})-3 <  \hat{m} <  4\mathbf{D}_1\log{(2\mathbf{D}_1\sqrt{\frac{\pi^2n^2}{6\delta} } )}+1.
$$
where 
$$
\mathbf{D}_1 = \frac{2(b+1)^2R^2}{\triangle_{ij}^2}  \ \ \& \  \ 
\mathbf{D}_2 = \frac{2(b-1)^2R^2}{\triangle_{ij}^2}.
$$
\end{lemma}

Given the lower and upper bounds of $\hat{m}$, Lemma \ref{the:split} proves that the outlier arm uses the less number of pulls to be apart from normal arms than the number needed for normal arms to be apart from its nearest normal arms.

\begin{lemma} \label{the:split}
Assume a set of arms $\{j\} \cup \mathcal{N}$, $j \not \in \mathcal{N}$, where $j$ satisfies:
\begin{equation} \label{eq:cond3}
 \forall i \in \mathcal{N},  \    \trimin(j, \mathcal{N} ) >  (1+\epsilon) 
 \dia(i, \mathcal{N}) 
\end{equation}
where $\epsilon > e^{\frac{1}{16}}-1$.
Then in pulling process, the round exists where $j$ is not the neighbor of any arms of $\mathcal{N}$ while the arms of $\mathcal{N}$ still belongs to a same community,  with probability at least 1 - $\delta$.
\end{lemma}

Based on Lemma \ref{the:split},   Lemma \ref{lemma:group} is to prove that the outlier arm group will form new communities separated from the communities formed by the corresponding normal arm groups.

\begin{lemma}\label{lemma:group}
Given an $(\epsilon, \rho)$-outlier arm group $\hatn$ with respect to $\sun$ and $\slown$ where $\epsilon > e^{\frac{1}{16}}-1$, then in the pulling process, the round exists where for each $j \in \hatn$, $j$ is not the neighbor of  any arms of $\sun$ and $\slown$ while the arms of $\sun$ still belongs to a community and  the arms of $\slown$ still belongs to a community. 
\end{lemma}

\subsection{Terminal Status of \sys }

The following theorem provides the bound regarding the number of pulls needed by \sys. In practice, \sys is an efficient algorithm, since it terminates when the assumed number of outlier arms achieve the terminal status, removing the unnecessary pulls on normal arms.

\begin{theorem}\label{theo:neighborar}
Give two arms $i, j \in \Psi $ and assume $\triangle_{ij} = |y_i - y_j| > 0$. If $i$ and $j$ can be pulled in infinite times, i.e., $m_i \rightarrow \infty, m_j \rightarrow  \infty$, then $i, j$ will not be neighbor arms in the end. 
\end{theorem}

\begin{proof}
Recall the Definition \ref{def:neigh}. $i, j$ are neighbor arms if
$$
|\hat{y}_i - \hat{y}_j| \leq b \left[ \beta_{i}(m_i, \delta') + \beta_{j}(m_j, \delta') \right ].
$$
Then we have
$$
\begin{aligned}
\beta(m_i, \delta') &= R\sqrt{\frac{-\log{\delta'}}{2m_i}} =  R\sqrt{\frac{\log{\frac{\pi^2n(T)^2}{6\delta }}}{2m_i}} \\
&=  R\sqrt{\frac{\log{\frac{\pi^2n}{6\delta }} + 2\log{T}}{2m_i}} \\
&=  R\sqrt{\frac{\log{\frac{\pi^2n}{6\delta }}}{2m_i} + \frac{\log{T}}{{m_i}}}
\end{aligned}
$$
Let $h(m_i) =  \frac{\log{\frac{\pi^2n}{6\delta}}}{m_i}$. Apparently $h(m_i)$ is monotonically decreasing.

In \sys, because the arms in  set $( \Psi - \hat{\Psi})$ are pulled in a round-robin way and $|\Psi - \hat{\Psi}| \leq n$, we can derive  $ n(m_i-1)< T < n(m_i + 1)$.  Since $T$ is an integer, we first suppose $T = n m_i $. Then we have $f(m_i) = \frac{\log{T}}{m_i} = \frac{\log{n m_i}}{m_i}$.  
The derivative of $f(m_i)$ is :
$$
f'(m_i) = -\frac{\ln{(nm_i)}-1}{m_i^2}. 
$$
Therefore  $f(m_i)$ is monotonically decreasing when $m_i > \frac{e}{n}$. For any $ n(m_i-1)< T < n(m_i + 1)$, we can use the similar way to prove  $f(m_i)$ is monotonically decreasing when $m_i$ is larger than a small constant.
If we can pull the arms infinite times, we have $\lim_{m_i \rightarrow \infty}  \beta(m_i, \delta') = 0$.  

Thus, $\lim_{m_i \rightarrow \infty, m_j \rightarrow  \infty }  b \left[ \beta_{i}(m_i, \delta') + \beta_{j}(m_j, \delta') \right ] = 0$. As $m_i \rightarrow \infty, m_j \rightarrow  \infty$, it has $\hat{y}_i = y_i, \hat{y}_j = y_j$. Because $\triangle_{ij} > 0$, it has $ lim_{ m_i \rightarrow \infty, m_j \rightarrow  \infty} |\hat{y}_i - \hat{y}_j| > 0$. Putting them together:
$$
\begin{cases}
lim_{ m_i \rightarrow \infty, m_j \rightarrow  \infty} |\hat{y}_i - \hat{y}_j| > 0 \\
\lim_{m_i \rightarrow \infty, m_j \rightarrow  \infty }  b \left[ \beta_{i}(m_i, \delta') + \beta_{j}(m_j, \delta') \right ] = 0
\end{cases}
$$
Therefore, $i$ and $j$ will be not neighbor arms if we can pull them infinite times.

\end{proof}

\begin{theorem}  \label{the:upperbound}
With probability at least $1- \delta$, the total number of pulls $T$ needed   to terminate for \sys is bounded by:
\begin{displaymath}
 T < 4 \mathbf{D}_3n( \log(2\mathbf{D}_3n) + \log \sqrt{ \frac{\pi^2n}{6\delta}}) + 2(n-1),
\end{displaymath}
where
$$
\mathbf{D}_3 = \frac{2(b+1)^2R^2}{\hat{\triangle}} 
$$
and $\hat{\triangle} = \min_{i, j \in \Psi, i \not = j} | y_i - y_j|$.
\end{theorem}

\begin{proof}
Consider two arms $i$ and $j$. When $i$ and $j$ still are neighbors, we have
$$
|\hat{y_i} - \hat{y_j}| \leq b [   \beta_i(m_i, \delta') + \beta_j(m_j, \delta')].
$$
With probability of at least 1-$\delta$,  $\mathcal{E}$ happens (Lemma \ref{lemma:event}). In the round $T$,  $\hat{m}_i$ achieves the maximal number of pulls on $i$ before $i$ and $j$ still are neighbors, and assume $\hat{m}_i = \min (\hat{m}_i, \hat{m}_j)$. $\triangle_{ij} = |y_i -y_j|$.    Based on Eq.\ref{eq:bound} in Lemma \ref{lemma:mbounds}, we can derive
$$
\begin{aligned}
\triangle_{ij} & \leq (b+1)(\beta_i(\hat{m}_i, \delta'(T)) + \beta_j(\hat{m}_j , \delta'(T)))\\
&=  2 (b+1) R\sqrt{\frac{-\log{\delta'(T)}}{2\hat{m}_i}}\\
&=  2 (b+1) R\sqrt{\frac{\log{\frac{\pi^2nT^2}{6\delta} }}{2\hat{m}_i}}\\
\Rightarrow \hat{m}_i & \leq \frac{  2 (b+1)^2 R^2\log{\frac{\pi^2nT^2}{6\delta}}}{ \triangle_{ij}^2}
\end{aligned}
$$
In round $T$, assume all the arms have not achieved the terminal status, we can obtain $T \leq n(\hat{m}_i+1)-1 \Rightarrow  \hat{m}_i \geq \frac{T+1}{n}-1 $. Hence,
$$
\begin{aligned}
\frac{T+1}{n}-1 \leq \hat{m}_i &\leq  \frac{  2 (b+1)^2 R^2\log{\frac{\pi^2nT^2}{6\delta}}}{ \triangle_{ij}^2}\\
(\frac{T+1}{n}-1) \triangle_{ij}^2 &\leq 2 (b+1)^2 R^2(\log{\frac{\pi^2n}{6\delta}} + 2\log{T}) \\
\frac{\triangle_{ij}^2}{n4(b+1)^2R^2}T -  \log T & \leq \frac{(n-1)\triangle_{ij}^2}{n4(b+1)^2R^2} + \frac{1}{2} \log \frac{\pi^2n}{6\delta}
\end{aligned}
$$
According to Lemma 8 in \cite{antos2010active}, we can derive
$$
\begin{aligned}
T&< 4 \mathbf{D}_1n( \log(2\mathbf{D}_1n) + \log \sqrt{ \frac{\pi^2n}{6\delta}}) + 2(n-1).
\end{aligned}
$$
Based on this inequality, the upper bound of $T$  is determined by $\triangle_{ij}$. 
Replace $\triangle_{ij}$ by $\hat{\triangle}$ and then prove this theorem.

\end{proof}



\section{Experimental Results}

With extensive experiments on both synthetic and real-world data sets, in this section, we evaluate the effectiveness and efficiency of \sys against existing baselines. 

\noindent \textbf{Baselines}. \xspace
We compare the proposed algorithms with the most related approach \cite{2017outlierarm} and the baselines it used.  \textit{RR} is the algorithm proposed by~\cite{2017outlierarm}, which pulls each arm in a round-robin way and terminates when there is no overlap between the confidence intervals of the given threshold and each arm. \textit{WRR} is a modified version of \textit{RR} to pull arms in a weighted round-robin way. Moreover, we add two more baselines, \textit{NRR} and \textit{IB}~\cite{2014combinatorial}, following the experiments in \cite{2017outlierarm}.

\noindent \textbf{Configurations.} \xspace
In the experiments, we set $\rho = 0.9$ and  $\delta = 0.1$ for all methods and data sets. For \sys, we measure its performance according to the ranking of  $\Omega$.
For \textit{NRR}, \textit{RR}, and \textit{WRR}, following \cite{2017outlierarm}, we set $k = \{2.0, 2.5, 3.0 \}$ for the $k$-sigma rule and report the best results of these baselines.

\subsection{ Data Sets}
In the evaluation, we use one synthetic data set and two real-world data sets.
For the arms of a data set, the expected rewards of all arms are known. Therefore, given $\epsilon$  and  $\rho$, we label the arm group as "outlier" if it satisfies the criteria of $(\epsilon, \rho)$-outlier arm group.

\noindent \textbf{Synthetic Data Sets}. \xspace 
We generated the synthetic data sets with different setting of $n$ and $\epsilon$, where $ n = \{ 20, 50, 100, 200, 400 \}$ and  $ \epsilon = \{ 2.5, 5 \}$. Moreover, we divided outlier arms into two types.
\begin{enumerate}
    \item \textbf{Upper-side outlier}: the injected outlier set, denoted by $\hatn$, is an ($\epsilon$, $\rho$)-outlier group with respect to $\sun$ and $\slown$ where $|\sun| = 0$ and $ |\slown| > (1-\rho)\times n$.
    \item \textbf{Intermediate outlier} : the injected $\hatn$ is an ($\epsilon$, $\rho$)-outlier group with respect to $\sun$ and $\slown$ where $|\sun|  >   \frac{(1-\rho)\times n}{2}  $ and $ |\slown| > \frac{(1-\rho)\times n}{2}  $.
\end{enumerate}

For each configuration, we injected the two types of outlier arms respectively.
There are 20 configurations in total. 
We first generated $\rho \times n$ normal arms with the expected reward randomly drawn from $[0, 1]$. Then, we injected  $(1-\rho)  \times n$ outlier arms until they satisfy the corresponding criterion of $(\epsilon, \rho)$-outlier arm group. 
Each reward obtained by pulling an arm follows a Bernoulli distribution, and thus $R$ = 1.

\noindent \textbf{Twitter}. \xspace 
It is a collected data set used in \cite{2017outlierarm} for the detection of outlier regions with respect to keywords from Twitter data. The data set contains 20 keywords and 47 regions with more than 5000 tweets, respectively. A multi-armed problem corresponds to a user looking to find the region where tweets have an exceptionally high probability of containing a specific keyword. Consider a keyword $w$. The 47 regions can be thought of as the arms with a Bernoulli distribution. In each round, we pull an arm to obtain a tweet from the corresponding region, and the reward is $1$ if the tweet contains $w$. Otherwise, the reward is $0$. Note that the expected rewards of each region for each keyword are available in this data set. And we use $(\epsilon = 2.5 , \rho = 90 \%)$ to generate the outlier labels.

\noindent \textbf{Yahoo! Today Module}. \xspace
This is a large-scale clickstream data set provided by Yahoo Webscope program. It contains 45,8811,883 user interaction events on Yahoo Today Module in a ten-day period in May 2009.
In this paper, we only focus on the interaction between user and F1  articles following \cite{2010licontextual}. An article can be regarded as an arm, where a pull on this article represents an interaction with a user. The reward received is $1$ if a click happens in this interaction. Therefore, the expected reward of an arm is the ratio of the number of clicks the corresponding article receives and the number of interactions it involves totally. This data set includes ten sub data sets of ten days, and we run all algorithms on the five sub data sets of May 02, 03, 05, 07, and 10. Also, in this data sets, we use $(\epsilon = 2.5 , \rho = 90 \%)$ to generate the outlier labels.

 \vspace{-0.5em}
\subsection{Performance Comparison}

For each test case, we run all algorithms $10$ times. Due to the limited space, we only report the average performance and the standard deviation. The metric "\textbf{correctness}" is the ratio of the number of correct results returned by an algorithm to the number of its total runs.

\noindent \textbf{Synthetic Data Sets}. \xspace
Figure \ref{synthetic1} presents the performance of each algorithm with varying $n$ and $\epsilon$.  \sys has the perfect performance on all synthetic data sets. For $RR$ and $WRR$, they achieve high accuracy on detecting Upper-side outliers, as shown in Figure \ref{synthetic1} (a) and (c), because they use the $k$-sigma rule to identify the arms with high expected rewards.
However, $RR$ and $WRR$ obtain the unsatisfactory performance on identifying intermediate outliers, as shown in Figure \ref{synthetic1} (b) and (d), since intermediate outliers cannot be distinguished from normal arms by exploring the mean reward and standard deviation.
In fact, the performance of $RR$ and $WRR$ changed drastically when using different $k$ on the same data set or setting the same $k$ on different data sets, because the standard deviation is influenced by the distribution of normal arms. Also, as the performance of $NRR$ and $IB$ is not satisfactory for detecting outlier arms, we will not report their efficiency next. In practice, the varying $\epsilon$ almost has no impact on the accuracy of \sys because it is adaptive to $\epsilon$. But all other algorithms were weakened as $\epsilon$ becomes smaller.

We also present the average cost (number of pulls) to terminate for each algorithm, as shown in Figure (\ref{synthetic2}). As we can see, \sys is the most efficient algorithm while maintaining perfect performance. Compared to $WRR$, \sys saves the 82\% cost on average, because \sys eliminates the unnecessary pulls on normal arms.

\begin{figure}[h]
\centering
\includegraphics[width =1.0\columnwidth]{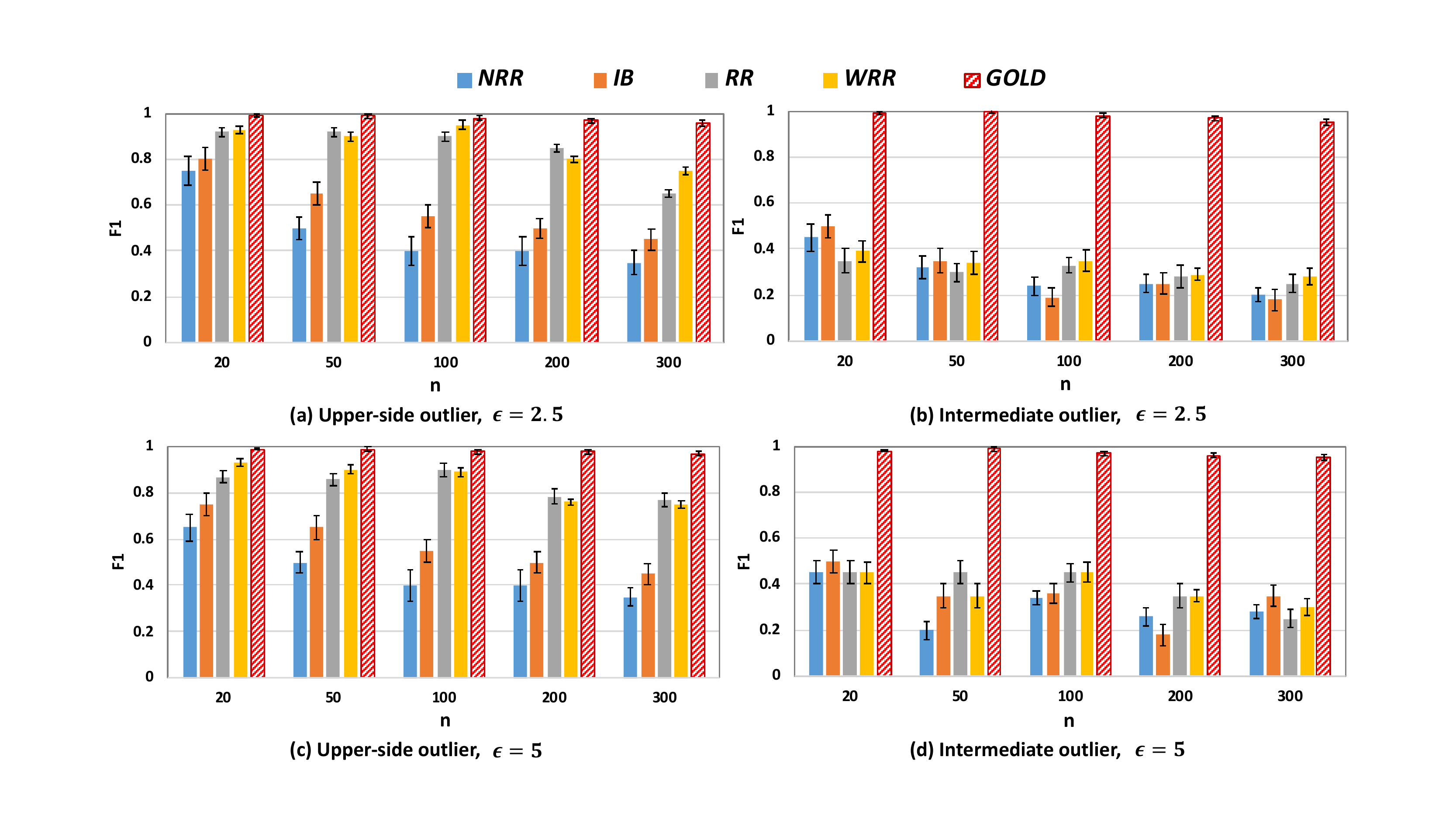}
 \vspace{-1.0em}
\caption{Accuracy comparison on synthetic data sets with 20 configurations.}
 \vspace{-1.0em}
 \label{synthetic1}
\end{figure}

\begin{figure}[h]
\centering
\includegraphics[width =1.0\columnwidth]{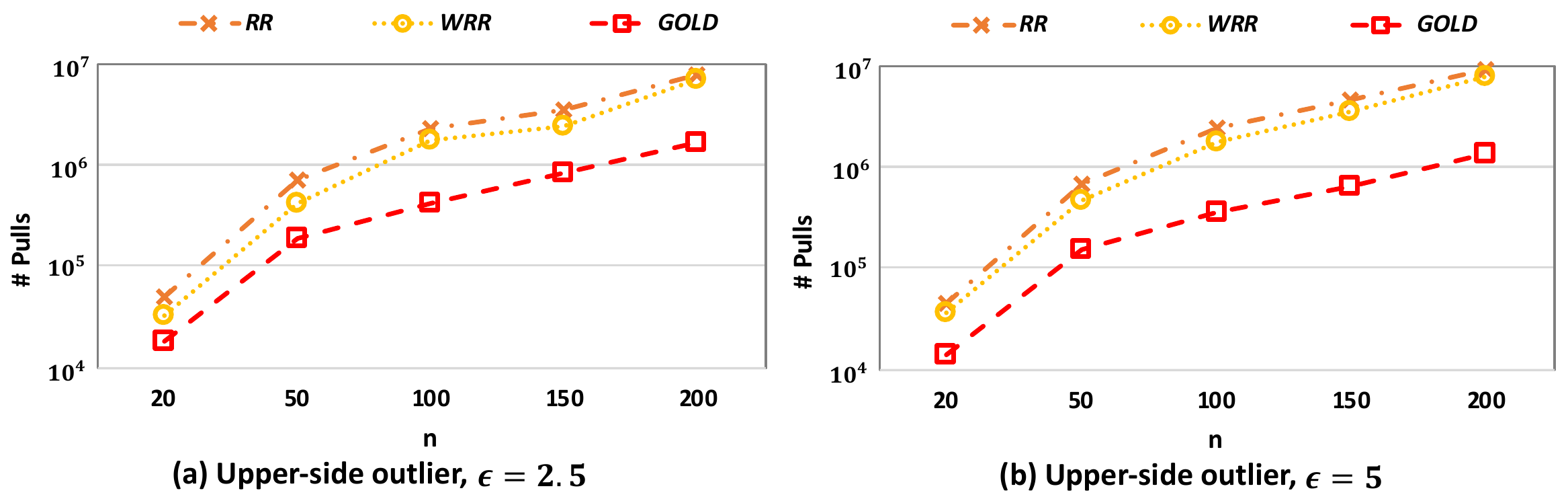}
 \vspace{-1.0em}
\caption{Efficiency comparison on synthetic data sets.} \label{synthetic2}
 \vspace{-1.0em}
\end{figure}

\begin{figure}[h]
\centering
\includegraphics[width =1.0\columnwidth]{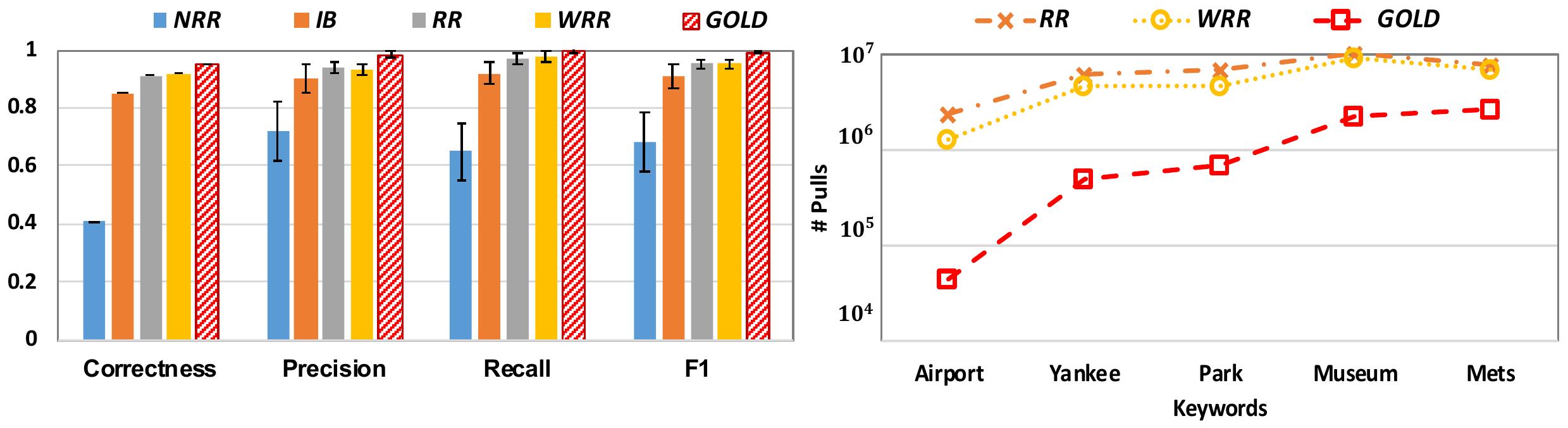}
 \vspace{-1.0em}
\caption{Effectiveness and efficiency on Twitter data set.} \label{twitter1}
\end{figure}

\noindent \textbf{Twitter}. \xspace
Figure \ref{twitter1} shows the performance of each method on this data set. As proven in the theoretical analysis, \sys achieves the required correctness, i.e., found the exactly correct outlier set with probability higher than $1-\delta$ (90\%).
As outlier arms all have exceptionally high expected rewards in this data set, $RR$ and $WRR$ also meet the correctness requirement. But $\sys$ has better empirical performance in terms of the F1 score. Since $\sys$ has a ranking for each outlier arm whereas $RR$ and $WRR$ only output the set of outlier arms, $\sys$ achieves higher precision with less false negatives. Unfortunately, the performance of $NRR$ is much worse than the rest.  $IB$ is a strong baseline, but the correctness is still less than $90 \%$.


Figure \ref{twitter1} also presents the average cost (the number of pulls) to terminate each algorithm. We choose $5$ keywords in the data set and report their efficiency.  \sys uses the minimal number of pulls compared to others while keeping almost perfect performance. \sys also is much faster than $WRR$ and $RR$ because the UCB of threshold in $WRR$ and $RR$ converges at a very slow rate.
In the case of keyword `Yankee', \sys saved 91\% of $WRR$'s cost and 93\% of $IB$ cost. Compared to $WRR$, \sys saved 89\% cost on average.

\begin{figure}[h]
\centering
\includegraphics[width =1.0\columnwidth]{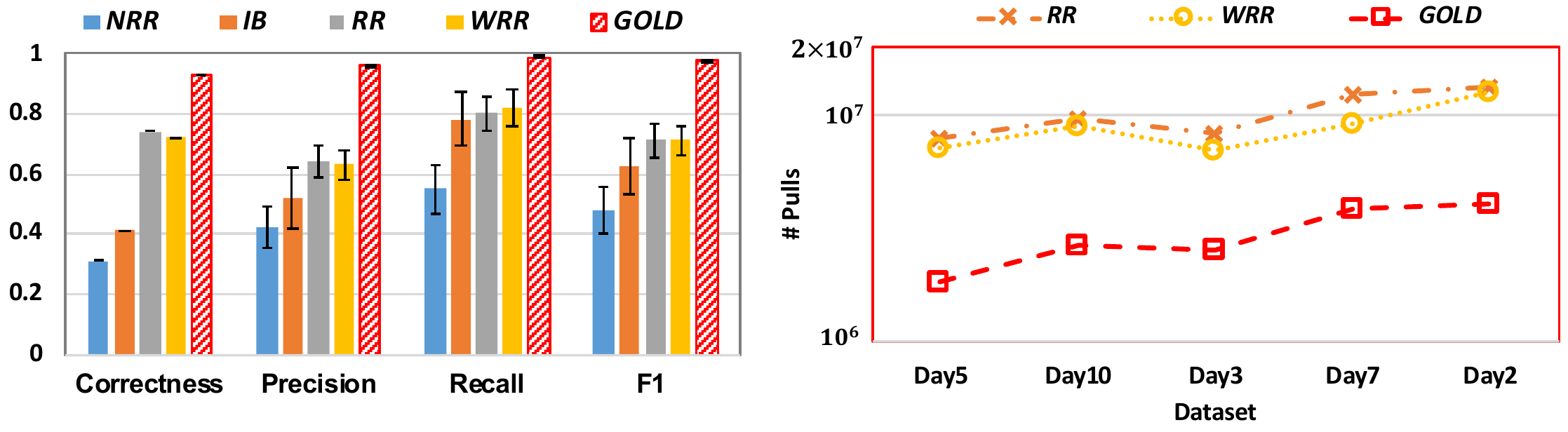}
 \vspace{-1.0em}
\caption{Effectiveness and efficiency on Yahoo data set.} \label{yahoo1}
\end{figure}

\noindent \textbf{Yahoo! Today Module}. \xspace Figure (\ref{yahoo1}) shows the results of all algorithms on Yahoo data set. It is easy to see that \sys achieves the required correctness (higher than $90\%$). On the contrary, $RR$ and $WRR$ fail to meet the correctness requirement because simply distinguishing outlier arms from normal arms by the mean plus standard deviation is likely to cause false negatives or false positives. 
For example, on May 03, even though we set $k=2$, some outlier arms are still missing; on May 07, even if we set $k=3$, some normal arms are still identified as outliers. Furthermore, the outlier arms on May 03 and 07 sub data sets contain some arms with exceptionally low expected rewards, which cannot be detected by $RR$ or $WRR$. The same as before, the performances of $NRR$ and $IB$ is much worse than the rest in this data set.

The right side of Figure (\ref{yahoo1}) exhibits the efficiency comparison of all algorithms on five days' sub data sets. $RR$ and $WRR$ terminate very slowly in this data set since the expected rewards of normal arms are very close to each other. However, the outlier arms are far from most of the other arms, which is leverage by \sys. Compared to the fastest baseline $WRR$, $\sys$ saves 78\% cost on average.

\vspace{-1.5pt}
\begin{figure}[h]
\centering
\includegraphics[width =0.6\columnwidth]{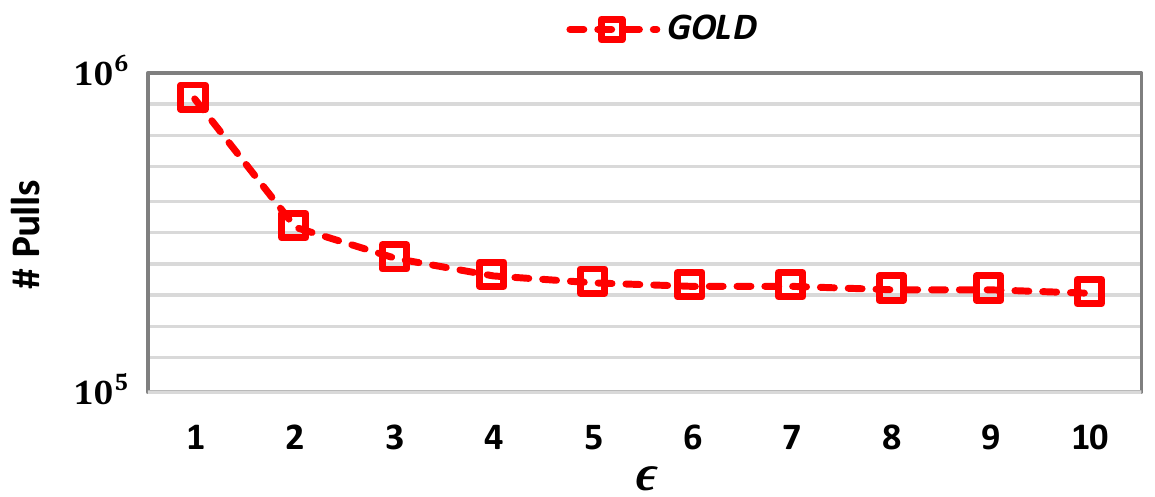}
\caption{Needed pulls by \sys\ with respect to $\epsilon$.} \label{eps}
 \vspace{-1.0em}
\end{figure}

\noindent \textbf{$\epsilon$ effect}. Figure \ref{eps} shows the number of pulls required by \sys\ with varying $\epsilon$ value. Here, we use the synthetic data set (Upper-side outlier, $n=100$). Consistent with our analysis in Theorem \ref{the:upperbound}, the number of pulls required by \sys decreases as $\epsilon$ increases.

 \vspace{-0.5em}
\section{Conclusion}
In this paper, we study a relatively new problem of identifying outlier arms in a MAB setting. Instead of applying a statistical rule, we propose new comprehensive definitions of outlier arms and outlier arm groups to identify the arms whose expected rewards are far from most of the other arms. This is widely applicable to many high-impact applications as compared to simply defining the outlier arms to be the ones with exceptionally high expected rewards. Moreover, we propose a novel algorithm named \sys, combining upper confidence bounds and graph features. Then, we analyze the properties of \sys from various aspects. Finally, we evaluate the empirical performance of our algorithm on both synthetic and real-world data sets, in comparison with state-of-the-art techniques. \sys achieves the near-perfect performance across all data sets using the least cost.


\hide{In the future, it will be challenging to extend the neighborhood distance to $k$-th nearest neighbor and explore $k$ while pulling the arms. The outlier arm group definition implicitly assumes that only one type of outlier exists in arms. It also is a challenge to formulate two types of outlier arms and prove the correctness of the proposed algorithm.}

\section{ ACKNOWLEDGMENTS}
This work is supported by National Science Foundation under Grant No. IIS-1947203 and Grant No. IIS-2002540, the U.S. Department of Homeland Security under Grant Award Number 17STQAC00001-03-03 and Ordering Agreement Number HSHQDC-16-A-B0001. The views and conclusions are those of the authors and should not be interpreted as representing the official policies of the funding agencies or the government.

\bibliographystyle{ACM-Reference-Format}
\bibliography{ref.bib}


\section{appendix}

The UCB $\beta$ of each arm follows the bounded distribution (Eq. (\ref{eq:boundeddis})), supposing the received reward is bounded within $[c, d]$ where $R = d-c.$ 

\begin{lemma} \label{lemma:UCB} (Upper Confidence Bound) 
According to Hoeffding's inequality, given a upper bound $\beta$, then the probability that difference between $\hat{y}_i$ and $y_i$ is larger than the bound can be bound as:
\begin{displaymath}
\mathbb{P}(|\hat{y}_i - y_i| > \beta) \leq 2 \exp{\frac{-2 m_i \beta^2}{R^2}}.
\end{displaymath}
Let $\delta' = 2 \exp{\frac{-2 m_i \beta^2}{R^2}}$, then we obtain:
\begin{displaymath}
\beta = \beta(m_i, \delta') = R\sqrt{\frac{-\log{\delta'}}{2m_i}}.
\end{displaymath}
\end{lemma}

 \vspace{-0.5em}

\textbf{Proof of Lemma \ref{lemma:event}}

\begin{proof}
Define the event
\begin{displaymath}
\mathcal{E}' = \{  \bigvee_{\forall T, \forall i}|\hat{y}_i - y_i| > \beta_{i}(m_i, \delta'(T))\}
\end{displaymath}
Then 
$$
\begin{aligned}
1 - \mathbb{P}(\mathcal{E}|\mathbf{I}) & = \mathbb{P}(\mathcal{E}'|\mathbf{I})\\
&\leq  \sum_{T=1}^\infty[\sum_{i=1}^n \mathbb{P}(|\hat{y}_i - y_i | > \beta_{i}(m_i, \delta'(T)))] \\
& \leq  \sum_{T=1}^\infty[\sum_{i=1}^n \delta'(T)]\\
&= \sum_{T=1}^\infty[\sum_{i=1}^n \frac{6\delta}{\pi^2 n T^2}]  = \frac{6}{\pi^2}\sum_{T=1}^\infty \frac{\delta}{T^2} = \delta\\
\Rightarrow \mathbb{P}(\mathcal{E}|\mathbf{I})& \geq 1-\delta.
\end{aligned}
$$
\end{proof}

\textbf{Proof of Lemma \ref{lemma:mbounds}}

\begin{proof}
We denote $|y_{i} - y_{j}|$ by $\triangle_{ij}$ . Then suppose $i$ and $j$ still are neighbor arms in round $T$, holding
\begin{equation} \label{eq:d}
|\hat{y_i} - \hat{y_j}| \leq b \left[  \beta_i(m_i, \delta') + \beta_j(m_j, \delta')\right]
\end{equation}
where $b>1$.
We can obtain
\begin{displaymath}
\begin{cases}
\hat{y}_i - \hat{y}_j \leq b \left[ \beta_{i}(m_i, \delta') + \beta_{j}(m_j, \delta')\right] , & \text{if} \ \hat{y}_i > \hat{y}_j \\
\hat{y}_i - \hat{y}_j \geq -  b \left[ \beta_{i}(m_i, \delta') + \beta_{j}(m_j, \delta') \right] , &  \text{otherwise}.\\
\end{cases}
\end{displaymath}
Based on Lemma \ref{lemma:event}, with probability 1-$\delta$, $\mathcal{E}$ happens. Then, we define a new event:
\begin{displaymath}
\mathcal{E} _{ij} = \{|\hat{y}_i - y_i|\leq \beta_{i}(m_i, \delta'(T)) \wedge |\hat{y}_j - y_j|\leq \beta_{j}(m_j, \delta'(T)), \forall T \}.
\end{displaymath}
Suppose $\mathcal{E}$ happens, and then $\mathcal{E}_{ij}$ must happen, deriving
$$
\begin{cases}
\hat{y}_i-\beta_i \leq y_i \leq \hat{y}_i + \beta_i \\
\hat{y}_j - \beta_j \leq y_j \leq \hat{y}_j + \beta_j. 
\end{cases}
$$
Then, assume $\hat{y}_i  > \hat{y}_j$ and we have
$$
\begin{aligned}
y_i - y_j &\leq \hat{y}_i + \beta_i - (\hat{y}_j - \beta_j)\\
 &= \hat{y}_i - \hat{y}_j + \beta_i + \beta_j \\
& \leq  b ( \beta_i + \beta_j)  + \beta_i + \beta_j\\
 & = (b+1)(\beta_i + \beta_j).
\end{aligned}
$$

Assume $\hat{m}_i = \min(\hat{m}_i, \hat{m}_j)$. Note that $\hat{m}_j \leq \hat{m}_i +1$ because we pull $i$ and $j$ respectively in an iteration.
Let $T+1$ be the round that is the first time when $i$ and $j$ are not neighbors. Suppose that the number of pulls on $i$ becomes $\hat{m}_i +1$.
It holds that
\begin{displaymath}
|\hat{y_i} - \hat{y_j}| > b [   \beta_i(\hat{m}_i+1, \delta') + \beta_j(\hat{m}_j, \delta')].
\end{displaymath}
Then,assume $ \hat{y}_i > \hat{y}_j$ and  obtain 
$$
\begin{aligned}
y_i - y_j &\geq \hat{y}_i - \beta_i - (\hat{y}_j + \beta_j)\\
 &= \hat{y}_i - \hat{y}_j - \beta_i - \beta_j \\
 & >  b ( \beta_i + \beta_j)  - (\beta_i + \beta_j) \\
 & = (b-1)(\beta_i(\hat{m}_i+1, \delta'(T+1)) + \beta_j(\hat{m}_j, \delta'(T+1))).
\end{aligned}
$$
Symmetrically, we can obtain the same result for $\hat{y}_i  < \hat{y}_j$. Thus, we can conclude that
\begin{equation} \label{eq:bound}
\begin{cases}
\triangle_{ij} \leq (b+1)(\beta_i(\hat{m}_i, \delta'(T)) + \beta_j(\hat{m}_j, \delta'(T)))\\
\triangle_{ij} > (b-1)(\beta_i(\hat{m}_i+1, \delta'(T+1)) + \beta_j(\hat{m}_j, \delta'(T+1)))
\end{cases} 
\end{equation}
Because $\hat{m}_i +1 \geq \hat{m}_j \geq \hat{m}_i $, we have
$$
\begin{aligned}
2(b-1) \beta_i(\hat{m}_i+1, \delta'(T+1)) &< \triangle_{ij} \leq 2(b+1) \beta_i(\hat{m}_i, \delta'(T)) \\ 
  2 (b-1) R\sqrt{\frac{-\log{\delta'(T+1)}}{2(\hat{m}_i+1)}} & < \triangle_{ij}\leq  2 (b+1) R\sqrt{\frac{-\log{\delta'(T)}}{2\hat{m}_i}}.
 \end{aligned}
$$
First, we have 
$$
\begin{aligned}
\triangle_{ij}\leq  2 (b+1) R\sqrt{\frac{\log{\frac{\pi^2nT^2}{6\delta} }}{2\hat{m}_i}}\\
\Rightarrow    \hat{m}_i  \leq 2 (b+1)^2 R^2  \frac{\log{\frac{\pi^2nT^2}{6\delta} }}{\triangle_{ij}^2}.
\end{aligned}
$$
In round $T$, assume all the arms have not achieved the terminal status, we have $T < n(\hat{m}_i+1)$. Then we have
$$
\begin{aligned}
\hat{m}_i & < \frac{ 2 (b+1)^2 R^2} {\triangle_{ij}^2}\log{\frac{\pi^2n^2(\hat{m}_i+1)^2}{6\delta} } \\
(\hat{m}_i+1) & < 2\mathbf{D}_1\log{(\hat{m}_i+1)} +    \mathbf{D}_1 \log{\frac{\pi^2n^2}{6\delta}}+1 .
\end{aligned}
$$

According to Lemma 8 in \cite{antos2010active}, we have
\begin{equation}
\hat{m}_i < 4\mathbf{D}_1\log{(2\mathbf{D}_1\sqrt{\frac{\pi^2n^2}{6\delta} } )}+1.
\end{equation}
Second, based on Eq.(\ref{eq:bound}), we have 
$$
\begin{aligned}
 \triangle_{ij} &> 2 (b-1) R\sqrt{\frac{\log{\frac{\pi^2n(T+1)^2}{6\delta }}}{2(\hat{m}_i+1)}}  \\
 \Rightarrow    \hat{m}_i  &> \frac{2 (b-1)^2 R^2}{{\triangle_{ij}^2}}  \log{\frac{\pi^2n(T+1)^2}{6\delta}}-1.
\end{aligned}
$$
Assume all the arms have not achieved the terminal status, we
have $T > n(\hat{m}_i-1)$. Then, we can obtain
$$
\begin{aligned}
\hat{m}_i  & > \frac{2 (b-1)^2 R^2}{{\triangle_{ij}^2}}  \log{\frac{\pi^2n^2 (\hat{m}_i-1)^2}{6\delta}}-1\\
\Rightarrow \hat{m}_i -1 &> 2\mathbf{D}_2\log{(\hat{m}_i-1)} + \mathbf{D}_2\log{\frac{\pi^2n^2}{6\delta}}-2.
\end{aligned}
$$

According to Lemma 8 in \cite{antos2010active}, we can derive
$$
\hat{m}_i > 4 \mathbf{D}_2 \log(2\mathbf{D}_2\sqrt{\frac{\pi^2n^2}{6\delta}})-3.
$$
Putting everything together, we can obtain
 \begin{equation}
 4 \mathbf{D}_2 \log(2\mathbf{D}_2\sqrt{\frac{\pi^2n^2}{6\delta}})-3 <  \hat{m}_i <  4\mathbf{D}_1\log{(2\mathbf{D}_1\sqrt{\frac{\pi^2n^2}{6\delta} } )}+1.
\end{equation}
We can obtain the same result assuming $\hat{m}_j = \min(\hat{m}_i, \hat{m}_j)$. There, it directly proves Lemma \ref{lemma:mbounds} if we replace $\hat{m}_i$ by $\hat{m} = \min(\hat{m}_i, \hat{m}_j)$.
\end{proof}
\textbf{Proof of Lemma \ref{the:split}}

\begin{proof}
Let $i = \arg\min_{i' \in \mathcal{N}, i' \not = j}|y_j - y_{i'}| $ and let $\hat{m}_j$ be the maximal number of pulls on $j$ when $i$ and $j$ still are neighbor arms. 
Given $\hat{m}_{j'} = \min(\hat{m}_j,\hat{m}_i)$, based on Lemma \ref{lemma:mbounds}, we can derive 
$$
 4 \mathbf{D}_2 \log(2\mathbf{D}_2\sqrt{\frac{\pi^2n^2}{6\delta}})-3 < \hat{m}_{j'} <  4\mathbf{D}_1\log{(2\mathbf{D}_1\sqrt{\frac{\pi^2n^2}{6\delta} } )}+1.
$$
Then, given any $k \in \mathcal{N}$, let $l =   \arg\min_{i' \in \mathcal{N}, i' \not = k}|y_k - y_{i'}|$.
Based on Eq. (\ref{eq:cond3}), it has  $\triangle_{ij} > (1 + \epsilon)\triangle_{kl}$ where $\triangle_{kl} = |y_k - y_l|$. Let  $\hat{m}_{k'} = \min(\hat{m}_k,\hat{m}_l)$. Based on Lemma \ref{lemma:mbounds}, we can derive 
$$
 4 \mathbf{D}'_2 \log(2\mathbf{D}'_2\sqrt{\frac{\pi^2n^2}{6\delta}})-3 <  \hat{m}_{k'} <  4\mathbf{D}'_1\log{(2\mathbf{D}'_1\sqrt{\frac{\pi^2n^2}{6\delta} } )}+1.
$$
If we want to prove $\hat{m}_{j'} < \hat{m}_{k'}$, we need to prove the upper bound of $\hat{m}_{j'}$ is less than the lower bound of $\hat{m}_{k'}$. Formally,
\begin{equation}\label{eq:ddd}
4 \mathbf{D}'_2 \log(2\mathbf{D}'_2\sqrt{\frac{\pi^2n^2}{6\delta}})-3 -[ 4\mathbf{D}_1\log{(2\mathbf{D}_1\sqrt{\frac{\pi^2n^2}{6\delta} } )}+1] > 0
\end{equation}
where
$$ 
\mathbf{D}'_2 =  \frac{2(b-1)^2R^2}{\triangle_{kl}^2}   \ \text{and} \ \mathbf{D}_1 = \frac{2(b+1)^2R^2}{\triangle_{ij}^2}.
$$
Thus, we have
$$
\begin{aligned}
\mathbf{D}'_2 \log(2\mathbf{D}'_2\sqrt{\frac{\pi^2n^2}{6\delta}}) - \mathbf{D}_1\log(2\mathbf{D}_1\sqrt{\frac{\pi^2n^2}{6\delta} } )& > 1 \\
\mathbf{D}'_2 (\log\mathbf{D}'_2 + \log2\sqrt{\frac{\pi^2n^2}{6\delta}})
 -  \mathbf{D}_1(\log\mathbf{D}_1 + \log2\sqrt{\frac{\pi^2n^2}{6\delta} } )& > 1\\
 \mathbf{D}'_2 \log\mathbf{D}'_2 -\mathbf{D}_1\log\mathbf{D}_1 + (\mathbf{D}'_2 -\mathbf{D}_1)\log2\sqrt{\frac{\pi^2n^2}{6\delta}}
&> 1\\
\end{aligned}
$$
To prove the inequality above, we prove
$$ 
 \mathbf{D}'_2 \log\mathbf{D}'_2 -\mathbf{D}_1\log\mathbf{D}_1 >1.
$$
Because $b>1$ and $R\geq\triangle_{ij}$, $\mathbf{D}_1>8$. We need to prove 
$$
\mathbf{D}'_2 \log\mathbf{D}'_2 -\mathbf{D}_1\log\mathbf{D}_1 > \frac{\mathbf{D}_1}{8} >1.
$$
This gives
$$
\begin{aligned}
\mathbf{D}'_2 \log\mathbf{D}'_2 -\mathbf{D}_1\log\mathbf{D}_1 - \frac{\mathbf{D}_1}{8}& > 0 \\
\mathbf{D}'_2 \log\mathbf{D}'_2 -\mathbf{D}_1 (\log \mathbf{D}_1 +\log e^{\frac{1}{8}}) &> 0 \\
\mathbf{D}'_2 \log\mathbf{D}'_2 -\mathbf{D}_1 \log( \mathbf{D}_1 e^{\frac{1}{8}}) &> 0 \\
\end{aligned}
$$
To hold the above inequality,  we prove $\mathbf{D}'_2  > \mathbf{D}_1 e^{\frac{1}{8}}$:
$$
\begin{aligned}
\frac{2(b-1)^2R^2}{\triangle_{kl}^2} & > \frac{2(b+1)^2R^2}{\triangle_{ij}^2}e^{\frac{1}{8}} \\
\triangle_{ij} & > \frac{b+1}{b-1}e^{\frac{1}{16}} \triangle_{kl}
\end{aligned}
$$
Replacing $b$ by $ \frac{1 +e^{\frac{1}{16}}+ \epsilon}{1 -e^{\frac{1}{16}} + \epsilon}$, we obtain 
$
\triangle_{ij}  > (1 + \epsilon) \triangle_{kl},
$
which completely follows the assumed condition.
Therefore Ineq.(\ref{eq:ddd}) is true and $\hat{m}_{j'} < \hat{m}_{k'}$, with the probability of at least $1- \delta$.

As for each $k \in \mathcal{N}$, $\triangle_{ij} \geq (1 + \epsilon) \triangle_{kl}$, similarly, we can prove  $\hat{m}_{j'} < \hat{m}_{k'}$ for any pair of $k$ and $l$. Therefore, in the round $T$, when $j$ is not the neighbor of any arms of $\mathcal{N}$, for any $k$ and $l$, they still are neighbor arms. 
Because $\trimin(j, \mathcal{N} ) >  (1+\epsilon) 
 \dia(k, \mathcal{N}), \forall k \in \mathcal{N}$,
it indicates $\mathcal{N}$ must be a connected component (i.e., arm community) in $G$ at round $T$ when the event $\mathcal{E}$ happened with probability of at least $1-\delta$. 
\end{proof}

\textbf{Proof of Lemma \ref{lemma:group}}

\begin{proof}
For each $j \in \hatn$, $j$ is an $(\epsilon, \rho)$-outlier arm with respect to $\sun$ and $\slown$. 
As $$
\begin{cases}
\ \forall i \in \sun,  \    \triangle_{\text{min}}(j,  \sun ) >  (1+\epsilon) 
 \dia(i, \sun) \\
 \ \forall i \in \slown,  \    \triangle_{\text{min}}(j,  \sun ) >  (1+\epsilon)
 \dia(i, \slown), \\
 \end{cases}
$$
based on Lemma \ref{the:split}, in the pulling process, the status exists that  $j$ is not the neighbor of any arms of $\sun$ while $\sun$ still belongs to a community denoted by $\mathcal{M}_1$ and $\slown$ still belong to a community denoted by $\mathcal{M}_2$.

As $$
\begin{cases}
 \ \forall i \in \sun,  \    \triangle_{\text{min}}(j, \slown) >  (1+\epsilon)  \dia(i, \sun),\\
\ \forall i \in \slown,  \    \triangle_{\text{min}}(j, \slown) >  (1+\epsilon)  \dia(i, \slown),\\
\end{cases}
$$
based on Lemma \ref{the:split}, in the pulling process, the status exists that  $j$ is not the neighbor of any arms of $\slown$ while $\sun$ still belongs to $\mathcal{M}_1$ and $\slown$ still belong to $\mathcal{M}_2$.

Therefore, we can drive that the status exists where  $j$ is not the neighbor of any arms of $\sun$ and $\slown$ while $\sun$ still belongs to $\mathcal{M}_1$ and $\slown$ still belong to $\mathcal{M}_2$, with the probability of at least $1-\delta$
\end{proof}

\textbf{Proof of Theorem \ref{theo1}}

\begin{proof}
Based on Lemma \ref{lemma:group}, for each $j \in \hatn$, $j$ is not the neighbor of  any arms of $\sun$ and $\slown$ while $\sun$ still belongs to a community $\mathcal{M}_1$ and  $\slown$ still belongs to a community $\mathcal{M}_2$. Assume in the $\hat{T}$, the status first happens. As $|\sun| + |\slown| > \rho \times n $ and $|\hatn| + |\sun| + |\slown| = n $ , it has $|\hatn|< (1-\rho)\times n $. According Line 31-34 of $\sys$, it has $\forall j \in \hatn, S[j] \leq \hat{T}$.

Because $|\sun| > (1-\rho)\times n$, if $|\sun| \not = 0$, all the arms of $\sun$ have not achieved the terminal status. Thus, we need keep pulling the arm of $\sun$ until it achieves the terminal status. It has $\forall i \in \sun, S[i] > \hat{T}$. Similarly, we can derive    $\forall i \in \slown, S[i] > \hat{T}$.
Putting everything together, with the probability of at least $1-\rho$, it has
$$
\begin{cases}
\forall j \in \hatn, S[j] \leq \hat{T}\\
\forall i \in  \sun \cup \slown, S[i] > \hat{T},
\end{cases}
$$
which directly proves the Theorem \ref{theo1} .
\end{proof}

\end{document}